\documentclass[runningheads]{llncs}
\usepackage{tikz}
\usepackage{eso-pic}

\usepackage{math}
\usepackage[T1]{fontenc}
\usepackage{xcolor}
\usepackage{graphicx}
\usepackage{hyperref}
\usepackage{color}

\usepackage{versions}

\makeatletter

\let\c@lemma\undefined
\let\c@proposition\undefined
\let\c@corollary\undefined
\let\c@definition\undefined
\let\c@example\undefined
\let\c@remark\undefined
\makeatother

\spnewtheorem{lemma}[theorem]{Lemma}{\bfseries}{\itshape}
\spnewtheorem{proposition}[theorem]{Proposition}{\bfseries}{\itshape}
\spnewtheorem{corollary}[theorem]{Corollary}{\bfseries}{\itshape}
\spnewtheorem{definition}[theorem]{Definition}{\bfseries}{\rmfamily}
\spnewtheorem{example}[theorem]{Example}{\bfseries}{\rmfamily}
\spnewtheorem{remark}[theorem]{Remark}{\bfseries}{\rmfamily}

 \includeversion{longversion}
 \excludeversion{shortversion}

\begin{document}
\title{Runtime Analysis of the $(\mu + 1)$-ES in a Homogenous Progress Model}
\titlerunning{Runtime Analysis of $(\mu + 1)$-ES}
%
\author{Johannes Lengler\inst{1}\orcidID{0000-0003-0004-7629} \and\\
Raghu Raman Ravi\inst{1}\orcidID{0000-0003-3641-1824}}
\authorrunning{J. Lengler, R. Ravi}
%
\institute{Eidgen\"{o}ssische Technische Hochschule Z\"{u}rich, Switzerland \email{\{johannes.lengler,raghu.ravi\}@inf.ethz.ch}.}
\maketitle              
\begin{abstract}
We introduce a new simple model to study the fitness progress of Evolution Strategies (ES) in generic problems. In this model, we bypass the underlying fitness landscape and assume that the mutation of any individual produces an offspring whose fitness relative to the parent is given by an invariant distribution $Z$, such as a mean-shifted Gaussian. This serves as a prototypical model for the optimisation landscape when an evolution algorithm operates far from the global optimum. This simple model can be used to approximate the optimisation process for problems where it is intractable to model the exact fitness function, including tasks such as hyperparameter tuning in machine learning models. 

We rigorously analyse the expected growth rate $\Rc_{\mu}$ of the continuous steady-state $(\mu+1)$-ES in this model. Unlike comma-selection strategies, the steady-state $(\mu+1)$-ES maintains overlapping generations, introducing complex mathematical dependencies among surviving parents that make it harder to analyse. We give a general technique to analyse the the $(\mu + 1)$-ES by constructing modified processes whose growth rates provably sandwich that of the original process. These modified processes are then easier to analyse but still close enough to the true process to give a tight bound on the expected growth rate. When $Z = \Nc(-\delta, 1)$ and $\mu \le e^{\delta}$, we show that $\Rc_{\mu} = \frac{\log^{1 + o(1)} \mu}{\mu} \Rc_1$.

\end{abstract}
\section{Introduction}

Evolutionary Algorithms are a prominent class of randomised optimisation heuristics in continuous search spaces intended to serve as a black-box optimisations tools. Inspired by the principles of biological evolution and natural selection, these algorithms iteratively maintain a population of candidate solutions, exploring the search space through stochastic mutation and exploiting favourable traits through a ``survival of the fittest'' selection mechanism. 

In the discrete setting, algorithms such as the $(1 + 1)$-EA, the $(\mu + 1)$-EA and the $(\mu + 1)$-GA have been theoretically analysed for various benchmarks ranging from simple theoretical functions such as \textsc{Onemax} and \textsc{LeadingOnes}~\cite{doerr2019theory} to complex classes of functions such as monotone functions~\cite{lengler2019general} or the SLO hierarchy~\cite{dang2024slo}. Likewise, in the continuous setting, Evolution Strategies (ES) have been studied for various benchmarks, ranging from simple sphere functions to complex classes like locally strong convex and Lipschitz continuous functions~\cite{morinaga2023convergence}. Traditionally, the main focus has been on the convergence rate close to the optimum, culminating in the recent proof of linear convergence for the state-of-the-art CMA-ES~\cite{hansen2001completely,hansen2016cma} by Armand Gissler~\cite{gissler2024linear}. In recent years, attention has also turned to non-asymptotic analysis of ES~\cite{akimoto2018drift} which yields insight into which improvements can be obtained far away from the optimum. Arguably, this is a domain that is very relevant for practice. For example, in tasks like hyperparameter tuning for machine learning, each fitness evaluation corresponds to a costly training process, and thus optimisation is not pursued deeply into the basin of attraction of a local or global optimum. Unfortunately, non-asymptotic analysis remains challenging, and results could only be obtained for very simple benchmark functions.

Based on his experience with such machine learning applications, Esteban Real recently suggested a simple new model for the optimisation progress in a fitness region of steady, but slow progress [personal communication]. The model bypasses the need for an explicit fitness function. 
%
%
%
Instead, he assumes that the fitness of each individual (which we call its \textit{phenotypical fitness} or simply \emph{phenotype}) can be described as the sum of its inherent, but unknown \emph{genotypical fitness} (or simply \emph{genotype}) and some noise term, where the noise reflects the inherent uncertainty of the fitness evaluation process. Both phenotype and genotype are real numbers. The key assumption is that the distribution of the genotypical fitness improvement $Z := g_{o}-g_{p}$ is \emph{translation-invariant} (i.e., independent of $g_{p}$), where $g_{p}$ is the genotypical fitness of the parent and $g_{o}$ is the genotypical fitness of the offspring. This captures the regime that slow, but steady progress can be achieved by mutation, which may be a reasonable assumption for progress while the optimisation process is not yet converging. 

Concretely, Real suggests to use a Gaussian distribution $\mathcal{N}(-\delta, 1)$ for $Z$, where the negative mean reflects the fact that a random mutation typically makes the genotype worse. Moreover, he also suggests to model the noise as another Gaussian $\mathcal{N}(0, \sigma^2)$. Hence, starting from a parent of genotype $g_{p}$, each offspring will have genotype $g_{o} = g_{p}+ \mathcal{Z}$ where $\mathcal Z \sim \mathcal{N}(-\delta, 1)$ and phenotype $p_{o} = g_{o} + \mathcal{M}$, where $\mathcal M \sim \mathcal{N}(0, \sigma^2)$. Note that only the phenotypical fitness is known to the algorithm, so selection is based on $p_{o}$, not on $g_{o}$. The quantity of interest for us then is the expected long term growth rate of the best genotype, which we denote by $\Rc_{\mu}$. We give a full formal description of the model in Section~\ref{sec:not_prelim}.

Interestingly, Real observed experimentally that for the $(\mu + 1)$-ES the progress rate is unimodal with respect to the population size $\mu$, meaning that there is a non-trivial optimal choice for $\mu$ that depends on the parameters $\delta$ and $\sigma$ of the two Gaussian distributions. He thus asks what the optimal population size is. We confirm his findings with our own experiments in Section~\ref{sec:simulations}. Unfortunately, in this work we cannot yet answer his question in the full model. Instead, we restrict ourselves to the noise-free case $\sigma =0$. In this case, we determine the expected growth rate as a function of $\mu$ and $\delta$, up to a factor of $O(\log \log \mu)$. The result suggests, confirmed by experiments, that the optimal population size in the noise-free case is $\mu=1$, for which we give an exact formula.


Although the model seems simple, the analysis of the $(\mu + 1)$-ES is surprisingly non-trivial, even in the absence of noise. We develop novel techniques to analyse $\Rc_{\mu}$ for general distributions of $Z$ (all in the absence of noise) and give tight bounds for it when the offspring distribution is given by a mean shifted Gaussian. We hope that those techniques can serve as a stepping stone for analyzing the general case with noise in future work. 

\subsection{Results and Proof Strategies}\label{sec:strategy}
In this paper, we provide a mathematically rigorous analysis of the expected progress rate of the $(\mu+1)$-ES in the noise-free model when the offspring distribution $Z$ is a mean-shifted standard Gaussian distribution $\mathcal{N}(-\delta, 1)$. For $\mu = 1$, the expected steady-state growth rate is exactly $\Rc_1(\delta) = \phi(\delta) - \delta\Phi^+(\delta)$, where $\phi$ and $\Phi^+$ are respectively the probability density function (PDF) and the complementary cumulative density function (CDF) of the standard normal distribution. 

For larger $\mu$, the growth rate scales roughly as $(\log \mu) /\mu$ while $\mu \le e^{\delta}$, and like $1/\mu$ afterwards. More precisely, let $m:= \min\{\mu,e^\delta\}$. Then we prove that the expected growth rate satisfies $\Rc_{\mu} = O(\tfrac{\log m}{\mu} \Rc_1(\delta))$ and $\Rc_{\mu} = \Omega(\tfrac{\log m}{\mu \log\log m} \Rc_1(\delta))$. Note that upper and lower bound match up to a factor of $O(\log \log m)$. In the following paragraphs, we will discuss the more interesting case $\mu \le e^{\delta}$, where $m=\mu$. The other case is simpler, but follows by similar arguments.

\paragraph{Intuition.} Our proofs are rigorous and self-contained, but they were guided by an intuition for the distribution that the genotypes in a population of size $\mu$ should approximatively maintain: we expect that the $i^{th}$ fittest individual in the population has fitness roughly $(\log i)/\delta$ smaller than the maximal fitness $g^*$ in the population. The reason is that from a typical parent, the probability of generating a fitness of at least $g^* - (\log i)/\delta$ is about $i$ times larger than the probability of generating a fitness larger than $g^*$. Crucially, this ratio is (approximatively) \emph{independent} of the exact fitness of the parent. 
Hence, in the range above $g^* - (\log i)/\delta$ we should expect $i$ times as many individuals than in the range above $g^*$, which translates into the $i^{th}$ fittest individual having fitness roughly $g^* - (\log i)/\delta$. This has a remarkable consequence: the rate by which the $i^{th}$ fittest individual generates an offspring fitter than $g^*$ scales like $1/(i\mu)$, where the factor $1/\mu$ stems from the rate $1/\mu$ at which each individual is picked for mutation. This translates into a growth rate of $\Rc_1(\delta)/(i\mu)$ from the $i^{th}$ fittest individual. Summing over all $i$ yields the growth rate of $\approx \Rc_1(\delta)\log\mu/\mu$. Note that this rate is by a factor of $\log \mu$ higher than the rate that only mutations of the fittest individual would generate. So most offspring that improve over the fittest individual are not direct offspring of that individual, but rather the rate is composed of non-trivial contributions of all potential parents in the population.

\paragraph{Upper Bound.} For the upper bound, we define a simpler process that always maintains a bounded gap of at most $(\log \mu)/\delta$ between the fittest and least fit individual, by artificially increasing the fitness to that level for weaker individuals. It is easy to show that the expected growth rate of this process strictly dominates that of the original process. Then, we introduce an exponential potential, $\Psi_t = \sum_{i = 1}^\mu e^{\theta g_i}$, based on the fitness values $g_i$, and bound the expected increase in this potential over time. Consequently, we bound the growth rate of the fittest individual, which translates into an upper bound for the original process.

\paragraph{Lower Bound.} For the lower bound, we again define a simpler process that always waits for the less fit individuals to catch-up before ``accepting'' an improvement to the best individual. More precisely, we wait until the $i^{th}$ fittest individual has fitness at least $g^* - (\log i)/\delta$, and in fact we artificially lower its fitness to that level. We show that this catch-up phase has a waiting time that scales like $\tfrac{\mu \log\log \mu}{\log\mu} $ with $\mu$. After this phase, we keep the position frozen until some offspring is fitter than $g^*$, which starts the next catch-up phase. The waiting time for this last step scales like $\tfrac{\mu}{\log\mu}$, so it is up to a $\log \log \mu$ factor identical to the waiting time of the catch-up phase. The waiting times then translate into a growth rate of $\tfrac{\Rc_1(\delta)\log\mu}{\mu \log\log \mu}$. We show that this process gives a lower bound on the growth rate of the original process. 

\paragraph{} We note that both these techniques can be applied to more general distributions than just the Gaussian. In particular, in the lower bound we prove a general formula that holds for any distribution of $Z$. Moreover, the ideas could potentially be applied to discrete search spaces, too.

\begin{shortversion}
\paragraph{Note to reviewers.} As customary for theory submissions, we had to omit some proofs due to the page limit. Reviewers can find the full paper at~\cite{raghu2026noisy} for their own discretion. They are not expected to consult or to review the full version. 
\end{shortversion}

\subsection{Related Work}

The study of expected improvement in fitness, termed progress rate, was highly popularised by Rechenberg ~\cite{rechenberg1978evolutionsstrategien} and further formalised by Beyer \cite{beyer2001theory}. Classical models for Evolution Strategies include the sphere model ~\cite{jiang2018improved} and quadratic functions ~\cite{jagerskupper20061+}. For the $(1+1)$-ES ~\cite{jagerskupper20061+,jiang2018improved,glasmachers2020global} and $(1 + \lambda)$-ES ~\cite{auger2005convergence,jagerskupper2006probabilistic}, progress rates are well-documented because offspring generation steps are stochastically isolated. Over the past two decades, the theoretical computer science community has developed rigorous tools—most notably Drift Analysis~\cite{lengler2019drift}—to prove exact expected runtimes and progress rates. J{\"a}gersk{\"u}pper \cite{jagerskupper20061+} provided some of the first rigorous bounds for the $(1+1)$-ES on the sphere. More recently, Akimoto, Auger, and Glasmachers \cite{akimoto2018drift} successfully applied drift theory to continuous domains, analyzing the expected hitting time of the $(1+1)$ -ES with success-based step-size adaptation.

The $(\mu + 1)$-ES remains sparsely studied - J{\"a}gersk{\"u}pper and Witt ~\cite{jagerskupper2005rigorous} provided the first rigourous analysis of the $(\mu + 1)$-ES on the sphere model with the $1/5$-success rule. In contrast, the effect of Population on Evolutionary algorithms has been thoroughly investigated in discrete search spaces ~\cite{witt2006runtime,lengler2019exponential,lengler2021runtime}.

Our model is also closely related to the notion of $N$-particle branching random walks ~\cite{berard2014limiting,brunet2006phenomenological,brunet2007effect}, which is a stochastic particle model for simulating wavefront propagation.

\section{Notation and Preliminaries}
\label{sec:not_prelim}

We denote by $\ind{E}$ the indicator random variable for the event $E$. We define $(x)_+ \equiv \max \bc{x, 0}$. We denote by $\mathcal{N}\bp{m, \sigma^2}$ a normal distribution with mean $m$ and variance $\sigma^2$. The functions $\phi\bp{\cdot}$ and $\Phi\bp{\cdot}$ denote, respectively, the PDF and the CDF of the standard normal distribution. Furthermore, we denote by $\Phi^+\bp{\cdot} \equiv 1 - \Phi\bp{\cdot}$ the complementary CDF of the standard normal.

We now formally describe the general Progress Model for $(\mu + 1)$ Evolutionary Algorithms. Consider a population of $\mu \in \N$ individuals. Each individual $i \in \bs{\mu}$ is associated with a \textit{genotype} $g_i \in \R$ and a \textit{phenotype} $p_i \in \R$, which represent the inherent fitness and the observed fitness of the individual. At each time step $t \geq 1$, an individual $i_t$ is selected uniformly at random from the population. This individual produces an offspring with genotype $g_{o} = g_{i_t} + Z$, where $Z$ is random variable supported on the reals. Additionally, the phenotype of this offspring is given by $p_{o} = g_{o} + \Mc$, where $\Mc$ is random variable supported on the reals. Following reproduction, the population temporarily increases to $\mu + 1$ individuals. The individual with the lowest phenotype is then discarded, breaking ties arbitrarily, returning the population size to $\mu$. Our objective is to analyse the evolution of the fittest individual's genotype over successive generations. In any generation $t$, we denote the genotype of the fittest individual as $g^*\bp{t}$. We define the expected growth rate $\Rc_{\mu}$ of the process as $\liminf_{t \to \infty} \frac{\E{g^*\bp{t}} - g^*\bp{0}}{t}$.

\begin{note}
    In this paper, we only study the noise-free version of the model, i.e., when $\Mc = 0$ and hence the genotypes and phenotypes of the individuals coincide. Moreover, we may identify an individual with its genotype, so we will treat $g_i$ as an individual. We believe that this is only a first step towards studying the model in general, which exhibits a much richer and arguably more realistic behaviour.
\end{note} 

We now introduce some general tools and lemmas that we use later in the analysis. We begin with a bound on the expecting waiting time for sampling sufficiently large values from a distribution. We use this lemma in the lower bound section.

\begin{lemma}
    \label{lem:sample_time}
    Let $N$ be a positive integer and for all $1 \leq i \leq N$, let $x_i$ be real numbers. Let $Z$ be a random variable and $F(\cdot)$ be a lower bound on its complementary CDF. Let $T = \max_i (i/F(x_i))$. Suppose that we draw independent samples of $Z$ until the following event, $E$, occurs: for every $1 \leq i \leq N$, there are at least $i$ samples of value at least $x_i$. Then, the expected waiting time for $E$ is at most $8T$.
\end{lemma}
\begin{proof}
    First, fix an arbitrary $1 \leq i \leq N$, and let $E_i$ be the event that there are at least $i$ samples with a value of at least $x_i$. Suppose that we draw $4T$ samples. The expected number of values greater than $x_i$ is at least $4T F(x_i) \geq 4 \cdot \frac{i}{F(x_i)} \cdot F(x_i) = 4i$. The event $E_i$ fails to occur if and only if the number of samples is less than $i$. Thus, by applying Chernoff bounds, the failure probability of the event $E_i$ is bounded by $\exp(-9i/8)$. 

    The failure probability of the global event $E = \cap_{i = 1}^N E_i$ when drawing $4T$ samples can be bounded via a union bound:
    \[ \Pr\bs{\neg E} \leq \sum_{i = 1}^N e^{-9i/8} \leq \sum_{i = 1}^{\infty} e^{-9i/8} = \frac{e^{-9/8}}{1 - e^{-9/8}} \leq \frac{1}{2}. \]
    Since the probability of success in any block of $4T$ samples is at least $1/2$, the expected waiting time for the event $E$ is at most $8T$. \qed
\end{proof}

\subsection{The Gaussian Distribution}

In the following, we give many bounds pertaining to the Gaussian distribution which use extensively in our proofs. 
\begin{shortversion}
    We defer the proofs of these lemmas to the full version of the paper.
\end{shortversion}

We begin with the following bound on the so called Mill's ratio, which is a direct consequence of the continued fraction representation given in ~\cite{shenton1954inequalities}.

\begin{lemma}
    \label{lem:mills_ratio_bounds}
    For all $x > 0$,
    \[ x < x + \frac{x}{x^2 + 2} < \frac{\phi\bp{x}}{\Phi^+\bp{x}} < x + \frac{1}{x}.\]
\end{lemma}

We next give some bounds on the ratios of $\Phi^+(.)$.

\begin{lemma}
    \label{lem:ccdf_log_concavity}
    For any $x > y > 0$, the following bounds hold:
    \begin{enumerate}
        \item $\frac{\Phi^+(x - y)}{\Phi^+(x)} \geq e^{xy - y^2/2}$
        \item $\frac{\Phi^+(x + y)}{\Phi^+(x)} \geq e^{xy - y^2/2} \left(\frac{x}{x + y}\right)$
    \end{enumerate}
\end{lemma}

\begin{longversion}
\begin{proof}
    We begin by noting that the derivative of the natural logarithm of $\Phi^+(t)$ is $\tfrac{d}{dt} \log \Phi^+(t) = \frac{-\phi(t)}{\Phi^+(x)}$. Then, by Lemma \ref{lem:mills_ratio_bounds}, we deduce that
    \[ -t - \frac{1}{t}< \frac{d}{dt} \log \Phi^+(t) < -t.\]

    \textbf{(Part 1):} 
    Integrating the right hand side of the inequality from $x - y$ to $x$ yields:
    \begin{align*}
    \log \Phi^+(x) - \log \Phi^+(x - y) & = \int_{x-y}^x \frac{d}{dt} \log \Phi^+(x) dt < \int_{x-y}^x -t dt = \left[ -\frac{t^2}{2} \right]_{x-y}^x \\
    & = -xy + \frac{y^2}{2}. 
    \end{align*}
    Rearranging the terms and exponentiating gives the strict lower bound:
    \[ \frac{\Phi^+(x - y)}{\Phi^+(x)} > e^{xy - y^2/2}. \]
    
    \textbf{(Part 2):}
    Similar to the previous part, we integrate the other side of the inequality from $x$ to $x + y$:
    \[ \log \frac{\xi(x + y)}{\xi(x)} > \int_{x-y}^x \left( -t - \frac{1}{t} \right) dt = \left[ -\frac{t^2}{2} - \log t \right]_{x}^{x + y} = -xy - \frac{y^2}{2} - \log\left( \frac{x + y}{x} \right).\]
    Exponentiating both sides yields the final upper bound:
    \[ \frac{\Phi^+(x + y)}{\Phi^+(x)} > e^{-xy - y^2/2} \left(\frac{x}{x + y}\right). \] \qed
\end{proof}
\end{longversion}

Next, we give the defining characterisation of the function $\xi(x) \equiv \phi\bp{x} - x \Phi^+\bp{x}$, which is central in our analysis.
\begin{lemma}
    \label{lem:integral_ccdf_normal}
    \[  \int_{x}^{\infty} \Phi^+\bp{z} dz = \int_{x}^{\infty} \bp{z - x}\phi\bp{z} dz = \phi\bp{x} - x \Phi^+\bp{x} \equiv \xi\bp{x}.\]
\end{lemma}
\begin{longversion}
\begin{proof}
    Let $\mathcal{I} \equiv \int_{x}^{\infty} \Phi^+\bp{z} dz$ denote the first integral. First, we integrate by parts to obtain
    \[\mathcal{I} = \bs{\Phi^+\bp{z} \cdot z}_x^{\infty} - \int_{x}^{\infty} \od{\Phi^+\bp{z}}{z} z dz = -x \Phi^+\bp{x} + \int_{x}^{\infty} \phi\bp{z} z dz, \tag{*}\]
    where we have used the fact that by definition $\od{\Phi^+\bp{z}}{z} = - \od{\Phi\bp{z}}{z} = -\phi\bp{z}$.

    To obtain the first equality, notice that $\Phi^+\bp{x} = \int_{x}^{\infty} \phi\bp{z} dz$ and hence from $(*)$,
    \[ \mathcal{I} = -x \int_{x}^{\infty} \phi\bp{z} dz + \int_{x}^{\infty} \phi\bp{z} z dz = \int_{x}^{\infty} \bp{z - x}\phi\bp{z} dz. \]
    
    To obtain the second equality, we substitute the definition of $\phi\bp{z}$ into $(*)$ to yield
    \[ \mathcal{I} = -x \Phi^+\bp{x} + \int_{x}^{\infty} \frac{e^{-z^2/2}}{\sqrt{2 \pi}} z dz.\]
    Next, we substitute $t = z^2/2$ with $dt = z dz$ in the integral to obtain
    \[ \mathcal{I} = -x \Phi^+\bp{x} + \bs{\frac{e^{-t}}{\sqrt{2 \pi}}}_{x}^{\infty} = \phi\bp{x} - x \Phi^+\bp{x},\]
    which completes the proof. \qed
    
\end{proof}
\end{longversion}

Similar to Lemma \ref{lem:ccdf_log_concavity}, we now prove some bounds on the ratios of $\xi(.)$.

\begin{lemma}
    \label{lem:xi_log_concavity}
    For any $x > y > 0$, we have:
    \[ \frac{\xi(x + y)}{\xi(x)} \leq e^{-xy - y^2/2} \left(\frac{x}{x + y}\right)^{2}.\]
\end{lemma}

\begin{longversion}
\begin{proof}
    We begin by evaluating the derivative of the natural logarithm of $\xi(t)$. Note that $\xi'(t) = \phi'(t) - \Phi^+(t) - t(-\phi(t))$. Since $\phi'(t) = -t\phi(t)$, the terms cancel, leaving $\xi'(t) = -\Phi^+(t)$.
    
    Thus, by the chain rule, the logarithmic derivative is:
    \[ \frac{d}{dt} \log \xi(t) = \frac{\xi'(t)}{\xi(t)} = -\frac{\Phi^+(t)}{\xi(t)} = \frac{\Phi^+(t)}{t \bp{\Phi^+(t)} - \phi(t)} = \frac{1}{t - \frac{\phi(t)}{\Phi^+(t)}}. \]

    Thus, using Lemma \ref{lem:mills_ratio_bounds}, we have
    \[ \frac{d}{dt} \log \xi(t) > -t - \frac{2}{t}.\]

    Integrating the inequality from $x$ to $x + y$ yields the following:
    \begin{align*} \log \frac{\xi(x + y)}{\xi(x)} & > \int_x^{x+y} \left( -t - \frac{2}{t} \right) dt = \left[ -\frac{t^2}{2} - 2 \log t \right]_x^{x+y}\\
    & = -xy - \frac{y^2}{2} - 2 \log\left( \frac{x}{x + y} \right).
    \end{align*}
    Exponentiating both sides yields the final upper bound. \qed
\end{proof}
\end{longversion}

\subsection{Single Individual}

In the following, we analyse the case where $\mu = 1$. Under this regime, the expected growth rate of the genotype of the single individual (which is always the fittest, trivially) can be computed explicitly.

\begin{lemma}
    \label{lem:growth_rate_single_individual}
    Suppose that $Z \sim \Nc(-\delta, 1)$. Then $\Rc_1 = \xi\bp{\delta}$.
\end{lemma}
\begin{proof}
    Suppose that at any given time step, the individual has a genotype of $g$. We compute the expected increase in this value during this step. By definition, the offspring produced has a genotype of $g - \delta + \Nc$ (where we use the short hand $\Nc \equiv \Nc\bp{0, 1}$). The improvement in a single step is, therefore, $\bp{\Nc - \delta}_+$, and thus the expected improvement is
    \[ \E\bs{\bp{\Nc - \delta}\ind_{\Nc \geq \delta}} = \int_{\delta}^{\infty} (x - \delta) \phi(x)dx = \xi\of{\delta},\]
    where in the last step, we have used Lemma \ref{lem:integral_ccdf_normal}. \qed
\end{proof}

\section{Upper Bound}

In the following, we establish an upper bound for the process. Our approach is to define an auxiliary process that is easier to analyse and whose growth rate provably dominates that of our original process. This modified process is nearly identical to the original, with one key exception: we continuously ensure that the genotype of each individual remains within a distance of $\Delta \geq 0$ from the fittest individual.

We now formally describe this modified process, denoted as $\bp{\mu + 1}\text{-}\textsc{ES}^{\text{high}}$. We maintain a population of $\mu \in \N$ individuals, each with a genotype $g_i \in \R$. To keep our initial analysis as general as possible, let $Z$ be a continuous random variable representing the mutation step. 

In each step, an individual is chosen uniformly at random. If its genotype is $g_p$, it produces an offspring with genotype $g_o = g_p + Z$, where $Z\sim\mathcal{N}(-\delta,1)$. We then select the $\mu$ individuals with the highest genotypes, breaking ties arbitrarily. Finally, if the new highest genotype is $g^*$, any individual with a genotype strictly less than $g^* - \Delta$ has its genotype artificially increased to $g^* - \Delta$.

We now establish a coupling between the $\bp{\mu + 1}\text{-ES}$ and $\bp{\mu + 1}\text{-}\textsc{ES}^{\text{high}}$ processes. Notice that at every time step $t \geq 1$, the transitions depend entirely on a uniform random choice $U_t \in \bs{\mu}$ and a mutation realisation $Z_t$. Both processes are driven exclusively by these shared random variables.

\begin{proposition}
    \label{prop:up_bound_coupling}
    Fix an arbitrary realisation of the random variables $\bc{U_t, Z_t}_{t \geq 1}$. At any time step $t \geq 1$, let $x_1\bp{t} \leq x_2\bp{t} \leq \dots \leq x_\mu\bp{t}$ be the genotypes of the $\bp{\mu + 1}\text{-ES}$ and $y_1\bp{t} \leq y_2\bp{t} \leq \dots \leq y_\mu\bp{t}$ be the genotypes of the $\bp{\mu + 1}\text{-}\textsc{ES}^{\text{high}}$. Then, we must have $x_i\bp{t} \leq y_i\bp{t}$ for all $i \in \bs{\mu}$.
\end{proposition}

\begin{longversion}
\begin{proof}
    We prove the claim by induction on $t$. Initially, the claim is true by construction. Suppose that the claim holds at the end of some time step $t - 1 \geq 0$. 

    Let $i = U_t$ be the individual chosen in this step. The genotype of the offspring in the $\bp{\mu + 1}\text{-ES}$ and the $\bp{\mu + 1}\text{-}\textsc{ES}^{\text{high}}$ is $x_o = x_i + Z_t$ and $y_o = y_i + Z_t$ respectively. We distinguish between four cases:
    \begin{itemize}
        \item If the offspring is not selected in either process, the induction holds trivially.
        \item If $y_o$ is selected but $x_o$ is not, the induction again holds trivially, as the values of $y_1, y_2, \dots, y_\mu$ cannot decrease due to selection.
        \item If $x_o$ is selected but $y_o$ is not, notice that $y_1\bp{t} \geq y_o \geq x_o$. Thus, the induction holds.
        \item Finally, suppose that both offspring are selected. Assume that this creates a position $j$ such that $x_j\bp{t} > y_j\bp{t}$. Notice that we must have $x_j\bp{t} \leq x_o$, as otherwise $x_j\bp{t} = x_j\bp{t - 1} \leq y_j\bp{t - 1} \leq y_j\bp{t}$. Thus, $y_j\bp{t} < x_j\bp{t} \leq x_o \leq y_o$. However, this implies $y_j\bp{t} = y_{j + 1}\bp{t - 1}$. Additionally, we must also have $x_j\bp{t} = x_{j + 1}\bp{t - 1}$, which leads to the contradiction $x_{j + 1}\bp{t - 1} > y_{j + 1}\bp{t - 1}$.
    \end{itemize}
    
    Finally, notice that at the end of each phase, we only (possibly) increase the genotypes of individuals that fall strictly below $y_\mu\bp{t} - \Delta$. This artificial boost strictly aids the majorisation and preserves the induction hypothesis. \qed
\end{proof}
\end{longversion}

Let $\Rc_{\mu}^{high}$ denote the expected growth rate of the $\bp{\mu + 1}\text{-}\textsc{ES}^{\text{high}}$. Then, we immediately get the following corollary.

\begin{corollary}
    \label{cor:up_bound_dominance}
    $\Rc_{\mu}^{high} \geq \Rc_{\mu}$.
\end{corollary}

To facilitate the upper bound analysis, we define an exponential potential function that captures the growth of the process. Such a potential function is frequently used in the analysis of $N$-particle branching random walks ~\cite{shi2015branching}. Consider the $\bp{\mu + 1}\text{-}\textsc{ES}^{\text{high}}$ process. At any time step $t \geq 1$, let $g_1\bp{t} \leq g_2\bp{t} \leq \dots \leq g_\mu\bp{t}$ be the sorted genotypes. For a fixed $\theta > 0$, we define:
\[ \Psi_t = \sum_{i = 1}^\mu e^{\theta g_i}.\]

We first prove a bound on the expected growth of this potential for a general mutation variable $Z$.

\begin{lemma}
    \label{lem:expected_potential_growth}
    Consider the $\bp{\mu + 1}\text{-}\textsc{ES}^{\text{high}}$. At any time step $t \geq 1$, we have $\E\bs{\Psi_t} \leq \Psi_0 \bp{1 + Q(\theta)}^t$, where
    \[ Q(\theta) = \frac{\bp{1 + \mu e^{- \theta \Delta}}}{\mu} \E\bs{e^{\theta Z} \ind{Z \geq - \Delta}}.\]
\end{lemma}
\begin{proof}
    Let us analyse the expected one-step increase in the potential, conditioned on its current state. Suppose that at the beginning of the time step $t \geq 1$, the individual $i$ is chosen. The new offspring is produced with genotype $g_o = g_i\bp{t - 1} + Z$.

    If $g_o < g_1\bp{t - 1}$, the offspring is rejected and the potential does not increase. If $g_o \geq g_1\bp{t - 1}$, the offspring is selected and the baseline increase in potential is at most $e^{\theta g_o}$. Additionally, if $g_o > g_\mu\bp{t - 1}$, this offspring becomes the new fittest individual. Consequently, any individual with a genotype strictly less than $g_o - \Delta$ is pulled up to $g_o - \Delta$. This artificial adjustment creates an additional potential increase of at most $\mu e^{\theta(g_o - \Delta)}$.
    
    Thus, the total increase in potential when the offspring is selected is at most:
    \[\bp{1 + \mu e^{- \theta \Delta}} e^{\theta g_o} = \bp{1 + \mu e^{- \theta \Delta}} e^{\theta g_i\bp{t - 1} + \theta Z}.\]
    Taking the expectation over the uniform random choice of $i$ yields the following:
    \begin{align*}
        \E\bs{\Psi_t - \Psi_{t - 1} \mid \Psi_{t - 1}, Z} &\leq \frac{1}{\mu}\sum_{i = 1}^\mu \bp{1 + \mu e^{- \theta \Delta}} e^{\theta g_i\bp{t - 1} + \theta Z}\\
        &= \frac{\Psi_{t - 1}}{\mu}\bp{1 + \mu e^{- \theta \Delta}} e^{\theta Z}.
    \end{align*}
    Next, we take the expectation with respect to the mutation variable $Z$, restricted to the domain where selection occurs (i.e., $g_o \geq g_1\bp{t - 1}$). This corresponds to the range $Z \geq g_1\bp{t - 1} - g_i\bp{t - 1}$. Because we are establishing an upper bound and the process strictly maintains a maximum gap of $\Delta$ between the highest and lowest genotypes, we can safely relax this integration range to $Z \geq -\Delta$. 
    
    Applying this relaxed range to the expectation over $Z$ gives the following:
    \[ \E\bs{\Psi_t \mid \Psi_{t - 1}} \leq \Psi_{t - 1} \bp{ 1 + \frac{\bp{1 + \mu e^{- \theta \Delta}}}{\mu} \E\bs{e^{\theta Z} \ind{Z \geq - \Delta}} }.\]
    Iterating this conditional expectation over all $t$ time steps yields the desired claim.\qed
\end{proof}

Next, we can use this to give a general upper bound for the process.

\begin{lemma}
    \label{lem:general_up_bound}
    For any $\theta > 0$, $\Rc_{\mu}^{high} \leq \frac{Q(\theta)}{\theta}$.
\end{lemma}
\begin{proof}
    From Lemma \ref{lem:expected_potential_growth}, we have
    \[\E\bs{e^{\theta g_\mu\bp{t}}} \leq \Psi_0 \bp{1 + Q(\theta)}^t \leq \exp\of{\log \Psi_0 + Q(\theta)t}.\]
    Now for every $A \geq 0$, applying Markov's inequality gives
    \[\P\bs{g_\mu\bp{t} \geq A + \frac{\log \Psi_0 + Q(\theta) t}{\theta}} \leq \frac{\E\bs{e^{\theta g_\mu\bp{t}}}}{e^{\theta A  + Q(\theta)t + \log \Psi_0}} \leq \exp\bp{-\theta A}.\]
    In other words, the random variable $g_\mu\bp{t} - \frac{Q(\theta)t + \log \Psi_0}{\theta}$ is sub-exponential with parameter $\theta$. Consequently, we have
    \[ \E\bs{g_\mu\bp{t}} \leq \frac{Q(\theta)t + \log \Psi_0}{\theta} + \frac{1}{\theta},\]
    which immediately implies the required lemma.\qed
\end{proof}

We are now equipped to establish the upper bound on the growth rate by explicitly setting our mutation to be Gaussian.

\begin{theorem}
    \label{thm:up_bound}
    Let $Z \sim \Nc(-\delta, 1)$ and $4 \leq \mu \leq \mu_{\max} = \lfloor e^{\delta} \rfloor$. Then,  $ \Rc_{\mu}^{high} \leq \frac{330\xi\bp{\delta}\log \mu}{\mu}$.
\end{theorem}
\begin{longversion}
\begin{proof}
    We set $\Delta \equiv \frac{\log \mu}{\delta}$. Consider any time step $t \geq 1$. Under the assumption that $Z \sim \mathcal{N}\bp{-\delta, 1}$, we evaluate $Q(\theta)$ from Lemma \ref{lem:expected_potential_growth}. It is not difficult to verify that for any $\theta > 0$ and for any real numbers $a < b$,
    \[\int_{a}^{b} e^{\theta z}\phi\bp{z} dz = e^{\theta^2 / 2}\bs{\Phi^+\bp{a - \theta} - \Phi^+\bp{b - \theta}}.\]
    
    Now we set the parameter $\theta$ of our potential function to $\theta = \delta(1 - \tfrac{6}{5\log \mu}) > 0$.  Substituting this into our expression with $\Delta = \tfrac{\log \mu}{\delta}$, we obtain:
    \begin{align*}
      Q(\theta) &\leq \frac{\bp{1 + \mu e^{-\log \mu + 6/5}} e^{-\delta^2 + \frac{6\delta^2}{5 \log \mu}}}{\mu}\E\bs{e^{\theta \Nc} \ind{\Nc \geq \delta - \Delta}} \\
      &\leq \frac{\bp{1 + e^{6/5}} e^{-\delta^2 / 2 + \frac{18\delta^2}{25 \log^2 \mu}}\Phi^+\bp{\frac{6\delta}{5\log \mu}  - \frac{\log \mu}{\delta}}}{\mu},
    \end{align*}
    where we slightly abuse the notation to denote by $\Nc$ a standard normal random variable. Now, by using Lemma \ref{lem:mills_ratio_bounds},
    \begin{align*}
      \Phi^+\bp{\frac{6\delta}{5\log \mu}  - \frac{\log \mu}{\delta}} &\leq \frac{1}{\sqrt{2 \pi} \bp{\frac{6\delta}{5\log \mu}  - \frac{\log \mu}{\delta}}} e^{- \frac{18 \delta^2}{25 \log^2 \mu} + \frac{6}{5} - \frac{\log^2 \mu}{2\delta^2}} \\
      &\leq \frac{5 e^{7/10} \log \mu}{\delta \sqrt{2 \pi}} e^{- \frac{18 \delta^2}{25 \log^2 \mu}},
    \end{align*}
    where in the last step, we have utilised the fact that $\frac{\delta}{\log \mu} \geq \frac{\log \mu}{\delta}$, and hence $\frac{6\delta}{5\log \mu} - \frac{\log \mu}{\delta} \geq \frac{\delta}{5\log \mu}$. Thus, we have:
    \begin{align*}
      Q(\theta) &\leq \frac{(1 + e^{6/5}) e^{-\delta^2 / 2 + \frac{18\delta^2}{25 \log^2 \mu}}}{\mu} \cdot \frac{5e^{7/10} \log \mu}{\delta \sqrt{2 \pi}} e^{- \frac{18 \delta^2}{25 \log^2 \mu}} \\
      &\leq  \frac{(43.5) \log \mu \phi(\delta)}{\mu\delta} \\
      &\leq \frac{(43.5)t \log \mu \phi(\delta)}{\mu\delta} \\
      &\leq \frac{44 \delta t \log \mu \xi(\delta)}{\mu},
    \end{align*}
    where the final inequality follows from the fact that $\phi(\delta) < (\delta^2 + 3) \xi(\delta)$, which in turn can be deduced from Lemma \ref{lem:mills_ratio_bounds}.
    
    Recall our choice of $\theta = \delta(1 - \tfrac{1.2}{\log \mu})$. For $\mu \geq 4$, $\log \mu \geq 1.386$, meaning $\tfrac{1.2}{\log \mu} \leq 0.866$. This implies $\tfrac{\delta}{\theta} \leq \tfrac{1}{1 - 0.866} \approx 7.46 \leq 7.5$. 
    
    Consequently, we conclude 
    \[\Rc_{\mu}^{high} \leq 7.5 \frac{Q(\theta)}{\delta} \leq \frac{330 \log \mu \xi\bp{\delta}}{\mu}.\]\qed
\end{proof}
\end{longversion}

We use the same approach on the original process to obtain a bound for all $\mu$, which in turn hold also for large $\mu$.

\begin{theorem}
    \label{thm:up_bound_large}
    Let $Z \sim \mathcal{N}\bp{-\delta, 1}$ and $\mu \geq 4$. Then $\Rc_{\mu}^{high} \leq \frac{4\xi\bp{\delta} \delta}{\mu}$.
\end{theorem}

\begin{proof}
    We can set $\Delta = \infty$ in Lemma \ref{lem:expected_potential_growth} to obtain 
    \[Q(\theta) = \frac{\E\bs{e^{\theta Z}}}{\mu} = \frac{e^{-\delta \theta} \E\bs{e^{\theta \Nc}}}{\mu} =\frac{e^{-\delta \theta + \theta^2 / 2}}{\mu}.\]
    Setting $\theta = \delta$, we further obtain
    \[Q(\delta) = \frac{e^{-\delta^2 / 2}}{\mu} = \frac{\sqrt{2\pi} \phi(\delta)}{\mu} \leq \frac{\sqrt{2 \pi} (\delta^2 + 3)\xi(\delta)}{\mu} \leq \frac{4\delta^2 \xi(\delta)}{\mu}.\]
    Finally, applying Lemma \ref{lem:general_up_bound} gives us the required result.\qed
\end{proof}

\section{Lower Bound}

Similarly to the upper bound, our approach is to define an auxiliary process that is easier to analyse and whose growth rate is provably dominated by our original process. We now formally describe the modified process, which we refer to as the $\bp{\mu + 1}\text{-}\textsc{ES}^{\text{low}}$. We remark that this is specifically for the noiseless case, and hence we omit the phenotypes of the individuals. We still have $\mu \in \N$ individuals and each individual $i \in [\mu]$ has an associated \textit{genotype} $g_i \in \R$. The process progresses in two alternating phases: the \textit{catch-up} phase and the \textit{improvement} phase. We also associate with the process parameters $0 = \Delta_1 \leq \Delta_2 \leq \dots \leq \Delta_{\mu}$. 

In the catch-up phase, the $\bp{\mu + 1}\text{-}\textsc{ES}^{\text{low}}$ process works exactly as the original process, except when there is an improvement to the fittest individual. If this happens, we artificially decrease the fitness of this new individual to that of the current fittest individual and simply continue until the end of the phase. The catch-up phase ends when the following condition is satisfied: for every $1 \leq i \leq \mu$, there are at least $i$ individuals with a genotype of at least $g^* - \Delta_i$, where $g^*$ is the genotype of the fittest individual. At the end of the catch-up phase, we artificially decrease the fitness of the individuals so that the $i^{th}$ fittest individual has a fitness of exactly $g^* - \Delta_i$. Without loss of generality, we assume that the initial configuration satisfies this condition. We will see later in the proof that the initial condition does not affect the asymptotic rate. 

In the improvement phase, we reject all offspring that do not improve the fittest individual. Thus, the population remains the same throughout the phase and the phase ends when there is a new fittest individual.

We will formulate most of our analysis for arbitrary values of $\Delta_i$, but set $\Delta_i = \log i / \delta$ in the end except for some truncation, see Theorems~\ref{thm:low_bound} and \ref{thm:low_bound_large} below for details. As explained in Section~\ref{sec:strategy}, this choice is natural because the probability to generate a fitness at least $g^*-\Delta_i$ is proportional to $i$. Hence, in the real process we should expect approximatively $i$ individuals of this fitness in the population, and this is what we aritficially enforce in the modified process.

In the next proposition, we establish a coupling between the two processes, demonstrating that the $\bp{\mu + 1}\text{-}\textsc{ES}^{\text{low}}$ process always maintains lower or equal genotypes compared to the coupled $(\mu + 1)-\textsc{ES}$. For each time step $t \geq 1$, we use a shared random choice of a uniform parent index, say $U_t$, and a shared sample from the mutation random variable, say $Z_t$.

\begin{proposition}
    \label{prop:low_bound_coup}
    Fix an arbitrary realisation of the random variables $\{U_t, Z_t\}_{t \geq 1}$. At any time step $t \geq 1$, let $x_1(t) \leq x_2(t) \leq \dots \leq x_{\mu}(t)$ be the ordered genotypes of the $(\mu + 1)-\textsc{ES}$ and $y_1(t) \leq y_2(t) \leq \dots \leq y_{\mu}(t)$ be the ordered genotypes of the $\bp{\mu + 1}\text{-}\textsc{ES}^{\text{low}}$. Then, we must have $x_i(t) \geq y_i(t)$ for all $1 \leq i \leq \mu$.
\end{proposition}
\begin{longversion}
\begin{proof}
    We prove the claim by induction on $t$. Initially, the claim is trivially true. Now, suppose that the claim is true at the end of some time step $t \geq 0$. The new offspring has a genotype of $x = x_{U_t} + Z_t$ and $y = y_{U_t} + Z_t$ in the two processes, respectively. We consider three main cases:
    \begin{itemize}
        \item If the offspring survives in the $\bp{\mu + 1}\text{-}\textsc{ES}^{\text{low}}$, but is discarded in the $(\mu + 1)-\textsc{ES}$, then it must be the case that $y_1(t) \leq y_{U_t} + Z_t$ and $x_1(t) \geq x_{U_t} + Z_t$. By the induction hypothesis, we know that $x_{U_t} \geq y_{U_t}$. Thus, $y_{U_t} + Z_t \leq x_1(t)$, and therefore the ordered claim holds for $t + 1$.
        \item If the offspring survives in the $(\mu + 1)-\textsc{ES}$, but is discarded in the $\bp{\mu + 1}\text{-}\textsc{ES}^{\text{low}}$, then the population of the $(\mu + 1)-\textsc{ES}$ improves, while the population of the $\bp{\mu + 1}\text{-}\textsc{ES}^{\text{low}}$ remains unchanged. Hence the inductive step goes through.
        \item If the offspring survives in both processes, let us assume that $x_{U_t} + Z_t = x_i(t + 1)$ and $y_{U_t} + Z_t = y_j(t + 1)$ for some $1 \leq i, j \leq \mu$. Recall that we must have $y_j(t + 1) \leq x_i(t + 1)$. Notice that for all $k > j$, we have $y_k(t + 1) = y_k(t) \leq x_k(t) \leq x_k(t + 1)$. Also, for all $k < \min(i, j)$, we have $y_k(t + 1) = y_{k + 1}(t) \leq x_{k + 1}(t) = x_k(t + 1)$. We consider two subcases:
        \begin{itemize}
            \item If $j < i$, we have $y_j(t + 1) \leq y_{j + 1}(t) \leq x_{j + 1}(t) \leq x_j(t + 1)$.
            \item If $j \geq i$, we have for all $i \leq k \leq j$ that $y_k(t + 1) \leq y_j(t + 1) \leq x_i(t + 1) \leq x_k(t + 1)$. 
        \end{itemize}
        \item If the offspring is discarded in both processes, the claim goes through trivially.
    \end{itemize}
    
    Finally, notice that we possibly decrease the genotype of the individuals at any step in the catch-up phase, and this only helps us. Thus, the induction hypothesis must be true for the time $t + 1$, which proves our claim.\qed
\end{proof}
\end{longversion}

Let $\Rc_{\mu}^{low}$ denote the expected growth rate of the $\bp{\mu + 1}\text{-}\textsc{ES}^{\text{low}}$. Then, we immediately get the following corollary.

\begin{corollary}
    \label{cor:low_bound}
    $\Rc_{\mu}^{low} \leq \Rc_{\mu}$.
\end{corollary}

We are now ready to investigate the expected growth rate of the $\bp{\mu + 1}\text{-}\textsc{ES}^{\text{low}}$.

\begin{lemma}
    \label{lem:low_bound_general}
    Let $\mu \geq 2$, and let $\mathcal{D}$ be a probability distribution on $\R$ such that for $I \sim \mathcal{D}$ the tail probability at a point $z$ is given by
    \[ \Pr\bs{I \ge z} = \frac{1}{\mu} \sum_{i = 1}^{\mu} \Pr\bs{Z > \Delta_i + z}.\]
    Then, $\Rc_{\mu}^{low} \geq \frac{\E\bs{I \cdot \ind{I \geq 0}}}{\E\bs{W(I) \cdot \ind{I \geq 0}} + 1}$, where 
    \[W(I) \equiv 8\mu \cdot \max_{2 \leq i \leq \mu} \bp{\frac{i}{\Pr\bs{Z > -\Delta_i} + \sum_{j = 1}^{\mu - 1} \Pr\bs{Z > I + \Delta_j - \Delta_i}}}.\]
\end{lemma}
\begin{proof}
    We consider a time frame that starts at the beginning of an improvement phase and ends at the end of the catch-up phase that immediately follows. Without loss of generality, we assume that the current fittest individual has a fitness of $0$. First, notice that any given step in the improvement phase, the probability that a new fittest individual is sampled is exactly $\Pr[I \geq 0]$, where $I \sim \mathcal{D}$. This is because we know the exact positions of all the individuals throughout the phase. Thus, the expected waiting time for this event is exactly $\frac{1}{\Pr[I \geq 0]}$.

    After the improvement phase ends, the new fittest individual has a fitness of $I$, where $I$ is sampled from $\mathcal{D}$ conditioned on being non-negative. Additionally, the $(i + 1)^{th}$ fittest individual now has a fitness of exactly $-\Delta_i$ for all $1 \leq i \leq \mu - 1$. Now, the catch-up phase immediately following this continues until the $(i + 1)^{th}$ fittest individual has a fitness of at least $I - \Delta_i$ for all $2 \leq i \leq \mu - 1$. Notice that at any given time step in this phase, the probability that an offspring with a genotype of at least $z$ is produced is lower bounded by 
    \[\frac{1}{\mu} \bs{\Pr\bs{Z > z - I} + \sum_{j = 1}^{\mu - 1} \Pr\bs{Z > \Delta_j + z}}.\]
    Thus, applying lemma \ref{lem:sample_time}, we conclude that the conditional expected waiting time for this phase to end is upper bounded by $W(I)$. Thus, the expected average progress made in this entire time period is lower bounded by
    \[\frac{\E\bs{I \mid I \geq 0}}{\E\bs{W(I) \mid I \geq 0} + \frac{1}{\Pr\bs{I \geq 0}}} = 
    \frac{\E\bs{I \cdot \ind{I \geq 0}}}{\E\bs{W(I) \cdot \ind{I \geq 0}} + 1},
    \]
    where in the last step, we have multiplied the numerator and the denominator with $\Pr[I\ge 0]$.
    Applying the Law of Large Numbers now concludes the Lemma.\qed
\end{proof}

We now give an explicit lower bound when the mutations follow an offset Gaussian.

\begin{theorem}
    \label{thm:low_bound}
    Let $3 \leq \mu \leq e^{\delta}$ and let $Z \sim \mathcal{N}(-\delta, 1)$. Then $\Rc_{\mu}^{low} \geq \frac{\xi(\delta)\log(\mu)}{770\mu(1 + \log \log \mu)}$.
\end{theorem}
\begin{proof}
    \begin{shortversion}
        We set $\Delta_i = \frac{\log i}{\delta} \leq 1$ for all $1 \leq i \leq \mu$ and apply Lemma \ref{lem:low_bound_general}. We then use Lemma \ref{lem:ccdf_log_concavity} and Lemma \ref{lem:xi_log_concavity} to appropriately bound the resulting expressions to obtain the result.\qed
    \end{shortversion}

    \begin{longversion}
    We set $\Delta_i = \frac{\log i}{\delta} \leq 1$ for all $1 \leq i \leq \mu$ and apply Lemma \ref{lem:low_bound_general}. We have
    \[W(I) \equiv 8\mu \cdot \max_{2 \leq i \leq \mu} \bp{\frac{i}{\Phi^+(\delta - \Delta_i) + \sum_{j = 1}^{\mu - 1} \Phi^+(\delta + I + \Delta_j - \Delta_i)}}.\]

    Notice that by our assumption on $\mu$, we have $\Delta_i \leq \delta$ for all $i$. Thus, we can use Lemma \ref{lem:ccdf_log_concavity} to deduce that:
    \begin{align*}
      \Phi^+(\delta + I + \Delta_j - \Delta_i) &> \Phi^+(\delta + I + \Delta_j) \cdot e^{(\delta + I + \Delta_j)\Delta_i - \Delta_i^2 / 2}\\
      &\geq \Phi^+(\delta + I + \Delta_j) \cdot e^{\delta \Delta_i - 1 / 2}\\
      &= \Phi^+(\delta + I + \Delta_j) \cdot i \cdot e^{-1/2}.
    \end{align*}
    Similarly, we can deduce that $\Phi^+(\delta - \Delta_i) \geq \Phi^+(\delta) \cdot i \cdot e^{-1/2}$. Thus, the waiting time is at most
    \[ W(I) \leq \frac{8 \sqrt{e} \mu}{\Phi^+(\delta) + \sum_{j = 1}^{\mu - 1} \Phi^+(\delta + I + \Delta_j)}.\]

    Next, we define $I_0 > 0$ to be the value such that $\Phi^+(\delta + I_0) \log \mu = \Phi^+(\delta)$. It is easy to verify $I_0 < \frac{2 \log \log \mu}{\delta}$. Consequently, using Lemma \ref{lem:ccdf_log_concavity} and the fact that $\delta \geq \log 3$, we have for all $0 \leq I \leq I_0$:
    \begin{align*}
         \Phi^+(\delta + I + \Delta_j) &> \Phi^+(\delta + I) \cdot e^{-(\delta + I)\Delta_j - \Delta_j^2 / 2} \cdot \bp{\frac{\delta + I}{\delta + I + \Delta_j}}\\
         &> \Phi^+(\delta + I) \cdot e^{-\delta \Delta_j - 1} \cdot \frac{1}{1.92}\\
         &> \Phi^+(\delta + I) \cdot \frac{1}{(1.92)j e}.
    \end{align*}
    Summing this over $j$ yields:
    \begin{align*}
         \sum_{j = 1}^{\mu - 1} \Phi^+(\delta + I + \Delta_j) &\geq \sum_{j = 1}^{\mu - 1}\Phi^+(\delta + I) \cdot \frac{1}{(1.92)j \sqrt{e^3}}\\
         &\geq \Phi^+(\delta + I) \cdot \frac{\log \mu}{(1.92)e}
    \end{align*}
    Consequently, we have
    \begin{align*}
        W(I) &\leq (15.4)\sqrt{e^3} \mu \cdot \min \bp{\frac{1}{\Phi^+(\delta)}, \frac{1}{\Phi^+(\delta + I)\log \mu}}.
    \end{align*}
    
    Recall that $\Phi^+(\delta + I_0) \log \mu = \Phi^+(\delta)$. Thus, we have
    \begin{align*}
        & \E \bs{W(I) \cdot \ind{I \geq 0}} \\
        & \qquad \leq \frac{(15.4)\sqrt{e^3} \mu}{\mu} \sum_{i = 1}^{\mu} \int_{I = 0}^{\infty} \min \bp{\frac{1}{\Phi^+(\delta)}, \frac{1}{\Phi^+(\delta + I)\log \mu}} \cdot \phi(\delta + I + \Delta_i)dI.
    \end{align*}
    We can split the integral into two parts, where $I \in [0, I_0]$ and $I \in (I_0, \infty)$ respectively. For the first part,
    \begin{align*}
        \int_{I = 0}^{I_0} \frac{\phi(\delta + I + \Delta_i)}{\Phi^+(\delta + I)\log \mu} dI &= \int_{I = 0}^{I_0} \frac{\phi(\delta + I)}{\Phi^+(\delta + I)\log \mu} \cdot e^{-(\delta + I)\Delta_i -\Delta_i^2 / 2} dI\\
        &\leq \frac{1}{i \log \mu} \cdot \int_{I = 0}^{I_0} \frac{\phi(\delta + I)}{\Phi^+(\delta + I)} dI\\
        &= \frac{1}{i \log \mu} \cdot \bs{-\log \Phi^+(\delta + I)}_{0}^{I_0}\\
        &= \frac{1}{i \log \mu} \cdot \log \bp{\frac{\Phi^+(\delta)}{\Phi^+(\delta + I_0)}}\\
        &= \frac{\log \log \mu}{i \log \mu}.
    \end{align*}
    For the second part, we have
    \begin{align*}
        \int_{I = I_0}^{\infty} \frac{\phi(\delta + I + \Delta_i)}{\Phi^+(\delta)} dI &= \int_{I = I_0}^{\infty} \frac{\phi(\delta + I)}{\Phi^+(\delta)} \cdot e^{-(\delta + I)\Delta_i -\Delta_i^2 / 2} dI\\
        &\leq \frac{1}{i} \cdot \int_{I = I_0}^{\infty} \frac{\phi(\delta + I)}{\Phi^+(\delta)} dI\\
        &= \frac{1}{i} \cdot \bp{\frac{\Phi^+(\delta + I_0)}{\Phi^+(\delta)}}\\
        &= \frac{1}{i \log \mu}.
    \end{align*}
    Thus, we can conclude that
    \begin{align*} \E \bs{W(I) \cdot \ind{I \geq 0}} & \leq 15.4\sqrt{e^3}(1 + \log \log \mu) \cdot \sum_{i = 1}^{\mu} \frac{1}{i\log \mu} \\
    & \leq (28.3) \sqrt{e^3}(1 + \log \log \mu),
    \end{align*}
    where we have used the fact that $\sum_{i = 1}^{\mu}\frac{15.4}{i\log \mu} \leq \sum_{i = 1}^{3}\frac{15.4}{i\log 3} \leq 28.3$ (since we have assumed that $\mu \geq 3$).
    
    Finally, with Lemma \ref{lem:xi_log_concavity} and using $e^{-\delta\Delta_i} = 1/i$ and the bounds $\Delta_i \le 1$ and $\delta \ge \log(3)$ in the fourth step, we have
    \begin{align*}
        \E\bs{I \cdot \ind{I \geq 0}} &= \frac{1}{\mu} \sum_{i = 1}^{\mu} \int_{I = 0}^{\infty}I \cdot \phi(\delta + I + \Delta_i)dI\\
        &= \frac{1}{\mu} \sum_{i = 1}^{\mu} \xi(\delta + \Delta_i)\\
        &\geq \frac{1}{\mu} \sum_{i = 1}^{\mu} \xi(\delta)e^{-\delta \Delta_i - \Delta_i^2/2}\cdot \bp{\frac{\delta}{\delta + \Delta_i}}^2\\
        &\geq \frac{1}{\mu} \sum_{i = 1}^{\mu} \frac{\xi(\delta)}{i}e^{ - 1/2}\cdot \frac{1}{3.65}\\
        &\geq \frac{\xi(\delta)\log \mu}{(3.65)\sqrt{e}\mu}.
    \end{align*}

    Thus, we can conclude that the expected growth rate is lower bounded by $\frac{\xi(\delta) \log \mu}{770\mu (1 + \log \log \mu)}$. \qed
    \end{longversion}
\end{proof}

\begin{theorem}
    \label{thm:low_bound_large}
    Let $\mu > e^{\delta}$ and let $Z \sim \mathcal{N}(-\delta, 1)$. Then $\Rc_{\mu}^{low} \geq \frac{\xi(\delta) \delta}{770\mu(1 + \log \delta)}$.
\end{theorem}
\begin{proof}
    We set $\Delta_i = \frac{\log i}{\delta} \leq 1$ for all $1 \leq i \leq e^{\delta}$, and $\Delta_i = \infty$ for all $i > e^{\delta}$. The rest of the proof is similar to Theorem \ref{thm:low_bound}.\qed
\end{proof}

\section{Experimental Results}\label{sec:simulations}

Evaluating the exact steady-state growth rate of the $(\mu+1)$-ES step-by-step is computationally intractable for environments with large threshold penalties $\delta$. Because the acceptance probability decays exponentially as the penalty increases, simulating the algorithm generation-by-generation would require millions or billions of evaluations per successful offspring. 

To bypass this without losing mathematical precision, we implemented a continuous-time Markov jump equivalent of the evolutionary process. Instead of simulating failures step-by-step, the algorithm mathematically ``fast-forwards'' directly to the next successful acceptance. Let $g_{min}(t)$ denote the fitness of the worst individual in the population at time $t$, and $g_j(t)$ denote the fitness of the $j$-th parent. For an offspring generated by the parent $j$ with mutation $Z \sim \mathcal{N}(-\delta, 1)$ to be accepted, the following condition must be satisfied : $\Nc(0, 1) > g_{\text{min}}(t) - g_j(t) + \delta$. We define this threshold as $T_j = g_{\text{min}}(t) - g_j(t) + \delta$. The fast-forwarding relies on the following exact sampling techniques:

\begin{enumerate}
    \item \textbf{Waiting Time Sampling:} The probability that the parent $j$ generates a successful offspring is $p_j = \Phi^+(-T_j)$. The total probability of generating a successful offspring in a single generation is $P_{acc} = \frac{1}{\mu} \sum_{j=1}^\mu p_j$. The number of failed evaluations skipped before a success is sampled exactly using an exponential distribution with rate $P_{acc}$.
    \item \textbf{Truncated Tail Sampling:} Once a success occurs, the responsible parent $j$ is selected proportionally to its conditional probability of success: $\Pr(\text{Parent } j \mid \text{Success}) = p_j / (\sum_{k=1}^\mu p_k)$. The accepted mutation step is then sampled exactly from the right tail of the normal distribution, conditioned on $Z > T_j$, using Marsaglia's technique~\cite{marsaglia1964generating}.
\end{enumerate}

This methodology completely eliminates the computational bottleneck of low acceptance rates, allowing us to empirically measure the exact steady-state growth rate across extreme parameter regimes. For each parameter configuration, the simulation is run from an initial state of all zeros until the maximum true genotype in the population crosses a massive threshold ($g > 1000$), ensuring that the empirical average stabilises in the steady state.

\subsection{Noiseless Translation-Invariant Regime}

To evaluate the predictive accuracy of our theoretical model, we first simulated the steady-state growth rate in the noiseless environment over a wide range of population sizes, up to $\mu = 1024$. We tested four different penalty environments: $\delta \in \{2.5, 5.0, 10.0, 20.0\}$. 

Our theorems establish that the growth rate is on the order of $\frac{\log^{1 + o(1)} \mu}{\mu}$. In addition to the fact that $\Rc_1 = \xi(\delta)$, we compare the experimental results with the most plausible theoretical growth rate of $H_{\mu}/\mu$, where $H_\mu = \sum_{i=1}^\mu \frac{1}{i} \approx \log n + 0.57721$. Additionally, to demonstrate structural convergence, we normalise the empirical growth rates $R_\mu$ by $\xi(\delta)$ and overlay them against the idealised normalised curve $H_\mu/\mu$.

\begin{shortversion}
\begin{figure}[htbp]
    \centering
    \includegraphics[width=0.45\textwidth]{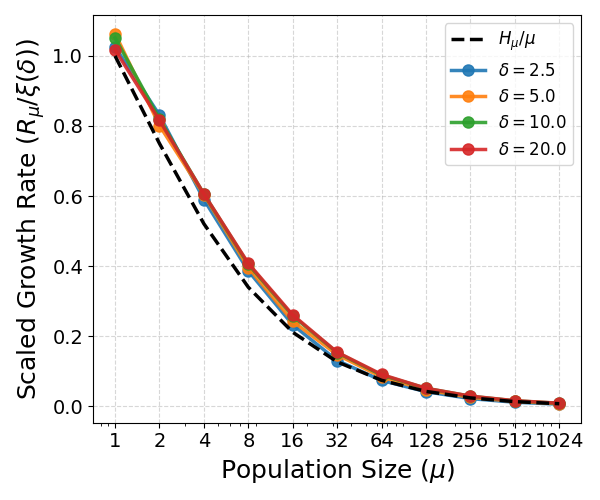}
    \includegraphics[width=0.45\textwidth]{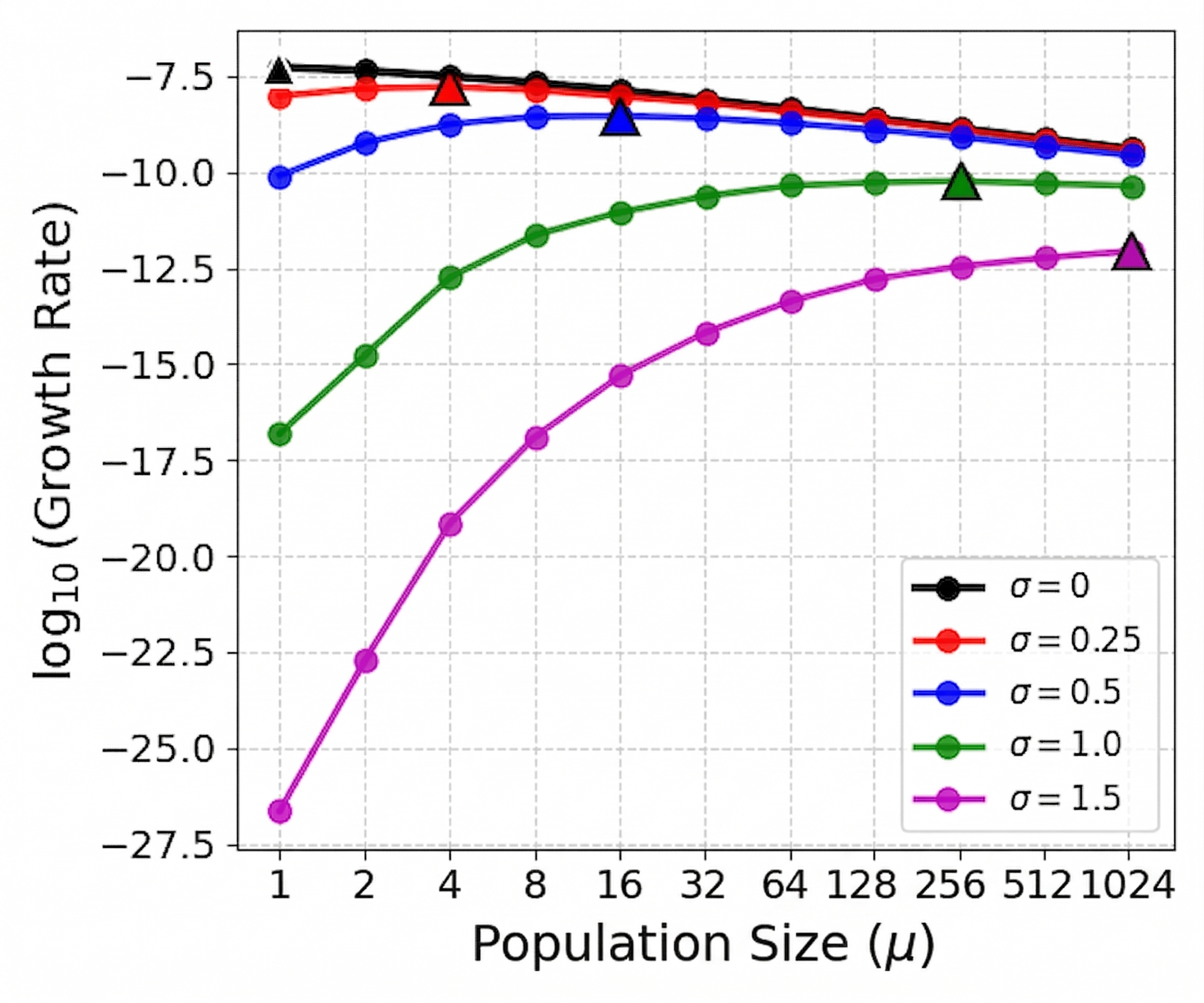}
    \caption{Scaled steady-state growth rate $R_\mu / \xi(\delta)$ of the $(\mu+1)$-ES across varying population sizes. \emph{Left:} The noiseless case $\sigma=0$. The empirical dynamics collapse perfectly onto the theoretical idealised bound $H_\mu/\mu$, regardless of the penalty severity $\delta$. \emph{Right:} Varying noise levels $\sigma$ for fixed $\delta = 5.0$ . While severe noise heavily penalizes small populations, larger populations naturally recover parallel asymptotic scaling. Maximum for each curve indicated by triangle.}
    \label{fig:noiseless_collapse}
\end{figure}
\end{shortversion}

\begin{longversion}
    \begin{figure}[htbp]
    \centering
    \includegraphics[width=0.85\textwidth]{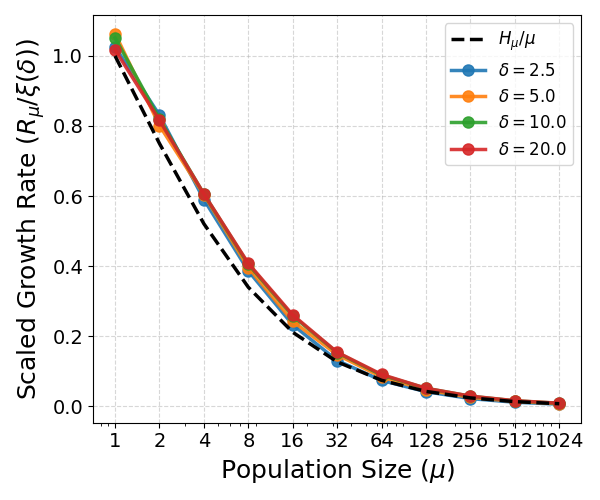}
    \caption{Scaled steady-state growth rate $R_\mu / \xi(\delta)$ of the noiseless $(\mu+1)$-ES across varying population sizes. The empirical dynamics collapse perfectly onto the theoretical idealised bound $H_\mu/\mu$, regardless of the penalty severity $\delta$. Maximum for each curve is indicated by a triangle.}
    \label{fig:noiseless_collapse}
\end{figure}
\end{longversion}

As shown in Figure~\ref{fig:noiseless_collapse} (left), the empirical results validate the theoretical framework. By scaling the growth rate by $\xi(\delta)$, the curves for all $\delta$ values collapse into a single universal trajectory that closely traces the idealised $H_\mu/\mu$ bound.

\subsection{Impact of Evaluation Noise}

We then extended our empirical evaluation to the noisy regime, where fitness evaluations are corrupted by Gaussian noise with variance $\sigma^2$. We fix the penalty at $\delta = 5.0$ and simulate the process for varying noise levels $\sigma \in \{0.5, 1.0, 2.0\}$, comparing them against the noiseless baseline ($\sigma = 0$).

\begin{longversion}
    \begin{figure}[htbp]
    \centering    
    \includegraphics[width=0.85\textwidth]{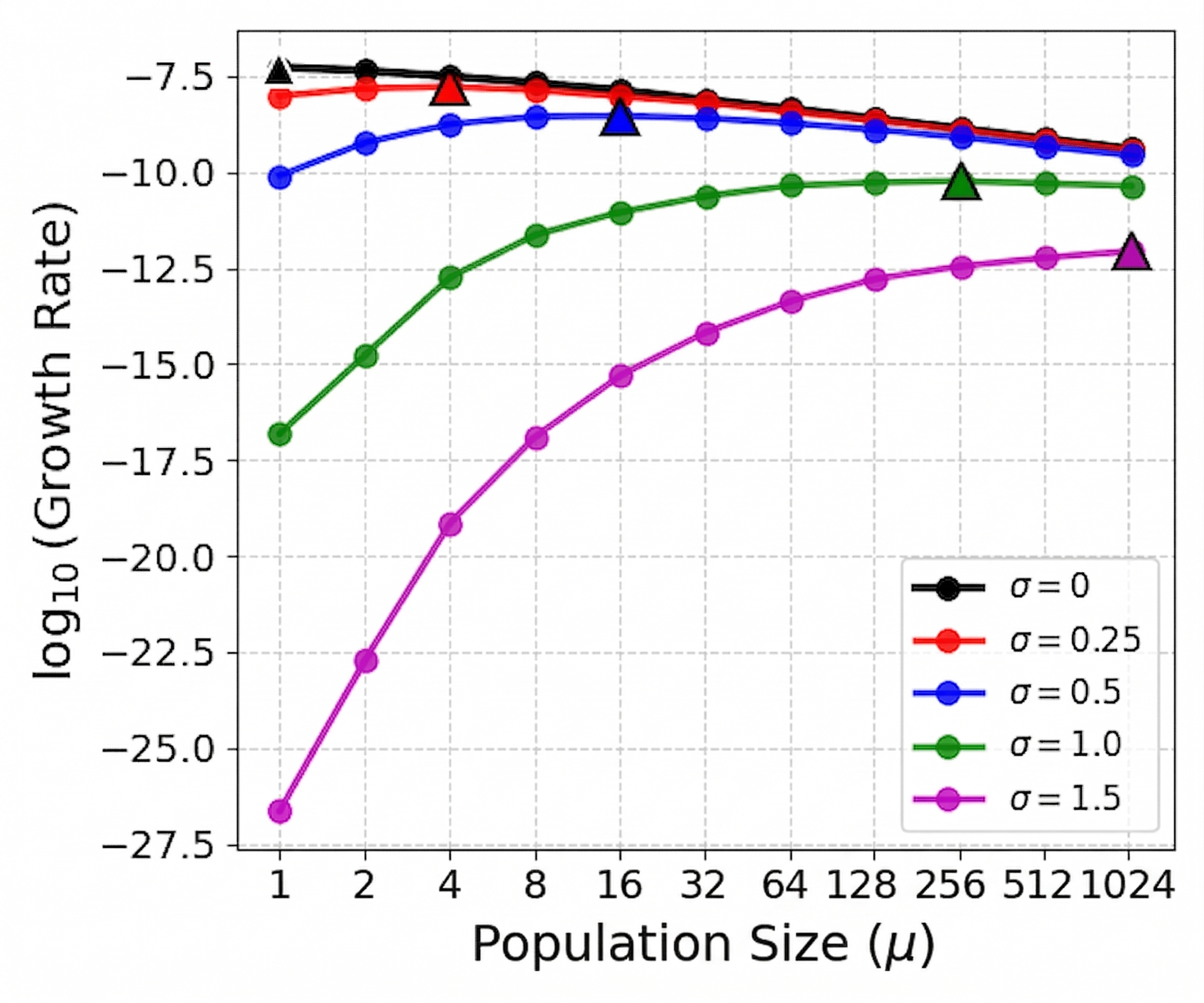}
    \caption{
Log-log plot of the steady-state growth rate vs. population size for $\delta = 5.0$ under varying noise levels $\sigma$. While severe noise heavily penalizes small populations, larger populations naturally recover parallel asymptotic scaling.}
    \label{fig:noisy_comparison}
\end{figure}
\end{longversion}

The results, plotted on a logarithmic scale in \begin{longversion}
    Figure~\ref{fig:noisy_comparison},
\end{longversion} 
\begin{shortversion}
    Figure~\ref{fig:noiseless_collapse} (right),
\end{shortversion} 
illustrate the non-monotonic dependence on the population size $\mu$. For small population sizes, large noise levels ($\sigma = 2.0$) trigger a catastrophic drop in the steady-state growth rate, spanning several orders of magnitude below the noiseless baseline. However, as the population size $\mu$ increases, the gap rapidly closes. This leads to a peak in the growth rate for a population size of $\mu > 1$. We also observe that this optimal $\mu$ is greater when the variance in the noise is larger.

\section{Conclusion}
In this paper, we have studied a new approach to model complex optimisation problems by representing fitness progress abstractly without explicit reference to a fitness function or a procedure for generating offspring. We investigate the noiseless case and develop novel approaches to upper and lower bound the expected growth rate in the fitness functions. When the offspring distribution is a mean shifted Gaussian $\Nc(-\delta, 1)$, we show that any non-trivial population $\mu \le e^{\delta}$ incurs a penalty factor of $\frac{\log^{1 + o(1)} \mu}{\mu}$. We also perform experiments to validate our results. Additionally, we also simulate the general noisy case to show that when there is non-zero noise, the optimal population size can be non-trivial.

The natural next step would be to investigate the noisy version of the model. Most intriguingly, does a non-trivial optimal population size exist, is it a true optimum in the long run or are we just observing a transient effect for small enough time periods? Moreover can we quantify the optimum analytically for a given offspring and noise distribution? One can also ask about the slow down in the growth rate due to noise, and to whether this slow down disappears for large enough population size.

\begin{credits}
\subsubsection{\ackname} This project was supported by the Swiss National Science Foundation [grant number 0003390]. The authors would also like to thank Esteban Real for the idea of the model.
\end{credits}


%
%
\bibliographystyle{splncs04}
\bibliography{refs}

\end{document}